%% file: main.tex
\documentclass[nohyperref]{article}

\usepackage{microtype}
\usepackage{graphicx}
\usepackage{subfigure}
\usepackage{float}
\usepackage{booktabs} 
\usepackage{thm-restate}

\usepackage{hyperref}

\usepackage[accepted]{icml2022}

\usepackage{amsmath}
\usepackage{amssymb}
\usepackage{mathtools}
\usepackage{amsthm}
\usepackage{soul}
\usepackage{graphics}

\usepackage[capitalize,noabbrev]{cleveref}

\theoremstyle{plain}
\newtheorem{theorem}{Theorem}[section]

\newtheorem{lemma}[theorem]{Lemma}

\theoremstyle{definition}

\theoremstyle{remark}

\usepackage[textsize=tiny]{todonotes}

\begin{document}

\twocolumn[
\icmltitle{Certified Adversarial Robustness via Anisotropic Randomized Smoothing}



\icmlsetsymbol{equal}{*}

\begin{icmlauthorlist}
\icmlauthor{Hanbin Hong}{yyy}
\icmlauthor{Yuan Hong}{yyy}

\end{icmlauthorlist}

\icmlaffiliation{yyy}{Department of Computer Science, Illinois Institute of Technology, Chicago, IL}

\icmlcorrespondingauthor{Hanbin Hong}{hhong4@hawk.iit.edu}
\icmlcorrespondingauthor{Yuan Hong}{yuan.hong@iit.edu}

\icmlkeywords{Machine Learning, ICML}

\vskip 0.3in
]




\begin{abstract}

Randomized smoothing has achieved great success for certified robustness against adversarial perturbations. Given any arbitrary classifier, randomized smoothing can guarantee the classifier's prediction over the perturbed input with provable robustness bound by injecting noise into the classifier. However, all of the existing methods rely on fixed i.i.d. probability distribution to generate noise for all dimensions of the data (e.g., all the pixels in an image), which ignores the heterogeneity of inputs and data dimensions. Thus, existing randomized smoothing methods cannot provide optimal protection for all the inputs. To address this limitation, we propose a novel anisotropic randomized smoothing method which ensures provable robustness guarantee based on pixel-wise noise distributions. Also, we design a novel CNN-based noise generator to efficiently fine-tune the pixel-wise noise distributions for all the pixels in each input. Experimental results demonstrate that our method significantly outperforms the state-of-the-art randomized smoothing methods.

\end{abstract}

\section{Introduction}
\label{sec: Introduction}
\input{Introduction}

\section{Related Work}
\label{sec: RelatedWork}
\input{RelatedWork}

\section{Isotropic Randomized Smoothing}
\label{sec: preliminary}
\input{preliminary}

\section{Anisotropic Randomized Smoothing}
\label{sec: ARS}

\input{ARS}

\section{Noise Generator}
\label{sec: NoiseGenerator}
\input{NoiseGenerator}

\section{Experiments}
\label{sec: Experiments}
\input{Experiments}

\section{Discussions}
\label{sec: Discussion}
\input{discussion}

\section{Conclusion}
\label{sec: Conclusion}
\input{Conclusion}

\newpage
\bibliography{ARS.bib}
\bibliographystyle{icml2022}

\newpage
\appendix
\onecolumn
\input{appendix}

\end{document}

%% file: Introduction.tex

\begin{figure*}[!h]
    \centering
    \includegraphics[width=160mm]{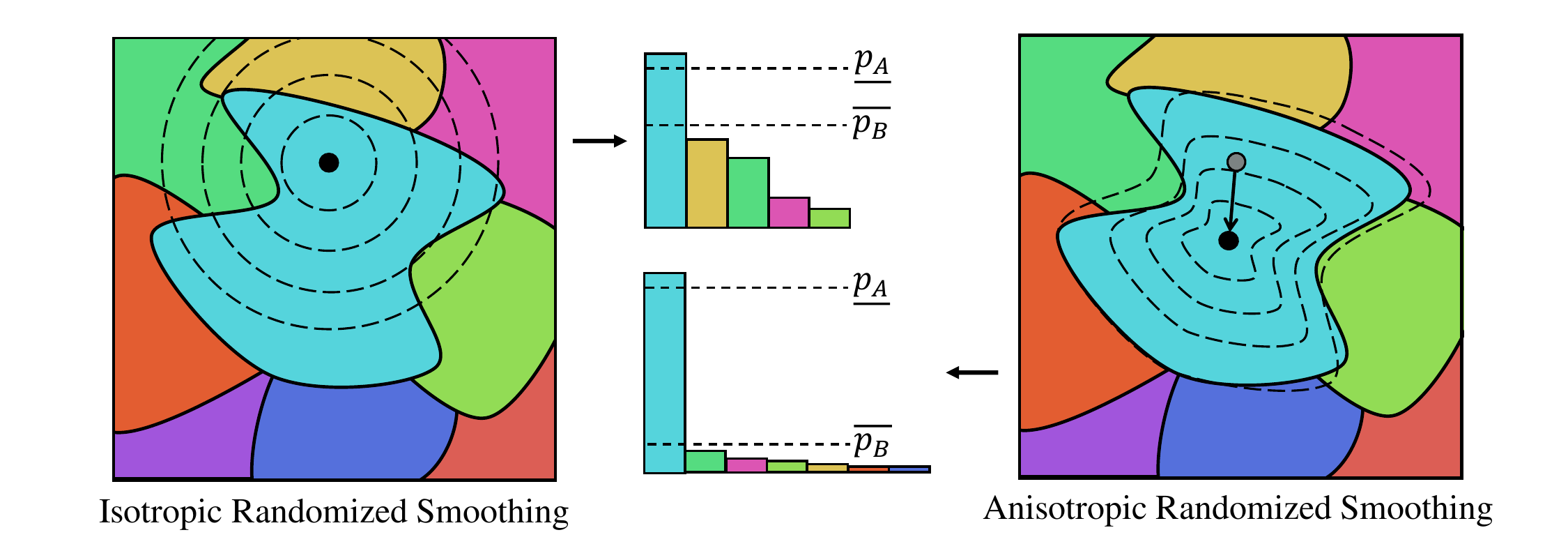}
    \caption{\textmd{Evaluation of the smoothed classifier at an input $x$. The decision regions of the base classifier $f$ are represented in different colors. The dashed lines are the level sets of the noise distribution adding to the input. The left figure illustrates the randomized smoothing with isotropic Gaussian noise $\mathcal{N}(0, \sigma^2 \mathbf{I})$ in \cite{cohen2019certified} whereas the right figure illustrates the randomized smoothing with anisotropic Gaussian noise $\mathcal{N}(\mu, \mathbf{\Sigma})$. The middle figure shows the prediction probabilities of the input over the noises. }}\vspace{-0.2in}
    \label{fig:overview}
\end{figure*}

Deep learning (DL) models have been proven to be vulnerable to well-crafted adversarial examples \cite{goodfellow2014explaining, carlini2017towards}. For example, adversaries can generate minor malicious perturbations either with or without access to the DL models (in white-box or blackbox settings). Once injected into the input of the DL models, it could trigger misclassification or misrecognition. These successful adversarial attacks are detrimental to DL models in real-world deployments and may cause severe consequences, e.g., car accidents in autonomous driving \cite{sun2020towards}, misdiagnosis in the auto-diagnosis \cite{ma2021understanding}, and misrecognizing faces \cite{dong2019efficient}. 

To protect the DL models against adversarial attacks, empirical defense methods have been proposed in the past decade. Through training more robust models by including adversarial examples in the training set \cite{madry2018towards,shafahi2019adversarial}, destroying the malicious perturbation \cite{xu2017feature,xie2018mitigating} or regularizing the features \cite{xie2019feature,yang2021adversarial}, these empirical methods have shown effective defenses against the adversarial attacks. However, given any new powerful defense method, stronger attacks \cite{athalye2018obfuscated,croce2020reliable,xie2019improving} will be designed to break the defenses. None of the empirical defenses can fully ensure the robustness of DL models all the time. Recently, certified robustness methods \cite{wong2018provable,cohen2019certified, lecuyer2019certified} were proposed to provide provable guarantees on the robustness of the DL models. They aim to certify whether potential adversarial perturbations can result in misclassification or not. Once certified, it guarantees that any perturbation cannot fool the classifier if it is within a boundary. Typically, this boundary is given by an $\ell_p$-norm ball, e.g., $\ell_1$, $\ell_2$, or $\ell_\infty$. 

The randomized smoothing (RS) methods \cite{lecuyer2019certified,teng2020ell,cohen2019certified} provide certified robustness on any arbitrary classifiers (compared to traditional certified methods on specific classifiers, e.g, ReLU based neural networks). By injecting the noises to the input, RS turns any arbitrary classifier into a smoothed classifier, then the robustness of the smoothed classifier can be guaranteed if the perturbation is within a theoretical bound in $\ell_p$-norm, i.e., $\emph{certified radius}$. For example, \cite{cohen2019certified} derives a tight $\ell_2$ certified radius for Gaussian noise. However, existing RS theories \cite{cohen2019certified,teng2020ell,yang2020randomized,zhang2020black} can only derive the certified radii for fixed i.i.d. noises, e.g., Gaussian noise \cite{cohen2019certified} or Laplace noise \cite{teng2020ell}, which applies identical distribution to different pixels and inputs. Thus, existing methods ignore the heterogeneity of the inputs and data dimensions, and cannot provide optimal protection for all the inputs. 


To pursue optimal protection for every input, we propose a novel randomized smoothing theory for anisotropic noise (to our best knowledge), which applies different distributions to generate noise for different data dimensions (e.g., image pixels). In this paper, we consider Gaussian noise as a use case to introduce anisotropic randomized smoothing (\emph{other noise can also achieve it with similar theories} as discussed in Section \ref{sec: Discussion}). We also propose a Noise Generator to generate the pixel-wise noise distributions for all the pixels in each input. Specifically, a tight certified radius is derived in our theory when all the pixels are smoothed by Gaussian noise with different means and variances. The Noise Generator uses a convolution neural network (CNN) to efficiently fine-tune the noise mean and variance for each pixel in randomized smoothing.


Compared to the traditional RS methods \cite{cohen2019certified,teng2020ell,zhang2020black}, our certified defense provides the following new significant benefits:

\begin{itemize}
    
\item \textbf{Higher Certified Accuracy}. We train the Noise Generator to generate the optimal means for the noises to be added to the input. The noise with proper means can move the input representation to the center of its class, e.g., some input located near the decision boundary can be adjusted to the class representation center by adding the noise mean towards the center (see Figure \ref{fig:overview} for illustration). This improves the certified accuracy for as high as $32.9\%$ on CIFAR10 and $20.6\%$ on ImageNet.  

\item \textbf{Larger Certified Radii}. We train the Noise Generator to also generate the optimal variance for the noises to be added to the input. Different from the isotropic Gaussian noise, we generate different variances for different pixels to keep the noisy sample within the decision boundary (see Figure \ref{fig:overview} for illustration). Thanks to the optimal means and variances, our smoothed classifier maintains a higher prediction accuracy over the same noise than the traditional smoothed classifier, which leads to larger certified radii. When certified accuracy is fixed at $20\%$, the certified radius can be improved from $1.10$ to $2.96$ on CIFAR10 and from $1.92$ to $3.73$ on ImageNet.

\item \textbf{Enhanced Robustness against Pre-Perturbing Attack}. We also study a new problem in randomized smoothing: what happens if the input is perturbed before certification with noise? Indeed, if the input is maliciously perturbed before injecting the noise for certification, the smoothed classifier's prediction could be guaranteed to be consistently wrong (certifying the class label of the perturbed input). We show that our method is more robust than the traditional RS methods against such adversarial attacks.

\end{itemize}




%% file: RelatedWork.tex
In this paper, all the defense methods are proposed against the evasion attacks to machine learning models. They aim to make the model correctly predict results on perturbed inputs. Typically, there are two types of defense methods: empirical defenses and certified defenses. The empirical defenses empirically protect the models while the certified defenses ensure the robustness of the models with provable guarantees.

\textbf{Empirical Defenses}. In the past decade, empirical defenses have been proposed to protect the machine learning models in different ways, e.g., training more robust models by including adversarial examples in the training data \cite{madry2018towards,shafahi2019adversarial,tramer2018ensemble, wong2019fast}, pre-processing the inputs to destroy the malicious patterns in the perturbation \cite{liu2019feature,xu2017feature,xie2019feature, samangouei2018defense}, regularizing the features in the model to eliminate the effects of perturbations \cite{xie2018mitigating,yang2021adversarial} or detecting the adversarial examples before fed into the model \cite{lu2017safetynet, metzen2017detecting, lee2018simple}. Although empirical evidence has shown that these methods can efficiently defend against adversarial attacks, none of them can guarantee model robustness against adversarial attacks. 

\textbf{Certified Defenses}. The certified defenses were proposed to guarantee robustness against adversarial perturbations. In general, the robustness can be guaranteed if the perturbations are within a boundary, e.g., a $\ell_1$ , $\ell_2$
or $\ell_\infty$ ball of radius $R$. The existing certified defenses can be roughly divided into two categories: exact certified defenses and conservative certified defenses. The exact certified defenses usually leverage satisfiability modulo theories \cite{katz2017reluplex,carlini2017provably,ehlers2017formal,huang2017safety} or mixed-integer linear programming \cite{cheng2017maximum,lomuscio2017approach,fischetti2018deep,bunel2018unified} to guarantee whether there exists a perturbation within radius $R$ or not. The conservative certified defenses provide conservative guarantee on the robustness by global/local Lipschitz constant methods \cite{gouk2021regularisation,tsuzuku2018lipschitz,anil2019sorting,DBLP:conf/icml/CisseBGDU17,DBLP:conf/nips/HeinA17}, optimization methods \cite{wong2018provable,wong2018scaling,raghunathan2018certified,dvijotham2018dual} or layer-by-layer certifying \cite{mirman2018differentiable,singh2018fast, DBLP:journals/corr/abs-1810-12715,DBLP:conf/icml/WengZCSHDBD18, DBLP:conf/nips/ZhangWCHD18}. In certain circumstances, it cannot provide the guarantee even when the malicious perturbation exists. However, the exact certified defenses cannot be scaled to large-size networks, and the conservative certified defenses usually assume specific types of networks, e.g., ReLU based networks. None of these schemes can provide certified robustness to any arbitrary classifiers until the randomized smoothing was proposed.

\textbf{Randomized Smoothing}. The randomized smoothing was first studied by Lecuyer et al. \cite{lecuyer2019certified}, where a loose theoretical bound for the perturbation is derived using Differential Privacy methods \cite{dwork2006differential, dwork2008differential}. The first tight guarantee was proposed by Cohen et al. \cite{cohen2019certified}, in which, any arbitrary classifier can be turned into a smoothed classifier by adding Gaussian noise to the data. The smoothed classifier's prediction can be guaranteed to be consistent within a certified radius in $\ell_2$-norm, which is tightly derived. Following the track of randomized smoothing, a series of methods have been proposed to guarantee the robustness against different $\ell_p$ perturbations with different noise distributions, e.g., Teng et al. \cite{teng2020ell} derives the certified radius for $\ell_1$ perturbations with Laplace noise, and Lee et al. \cite{DBLP:conf/nips/LeeYCJ19} derives the certified radius against $\ell_0$ perturbations with uniform noise. Some methods propose unified theories to guarantee the robustness against a diverse set of $\ell_p$ perturbations with different noises. For example, Zhang et al. \cite{zhang2020black} propose a framework from the optimization perspective to certify the robustness against $\ell_1$, $\ell_2$ and $\ell_\infty$ perturbations with special noise distributions. Yang et al. \cite{yang2020randomized} propose two different methods, e.g., level set method and differential methods, that can derive the upper bound and the lower bound of the certified radius in different norms for a wide range of distributions. However, all the existing randomized smoothing methods add noise drawn from a fixed distribution, e.g., Gaussian or Laplace, to all the inputs and all dimensions of each input (e.g., all the pixels on an image). This ignores the heterogeneity of inputs and even the pixels. Thus, they cannot provide the optimal protection for every input and pixel.

Therefore, we establish the a novel randomized smoothing method based on anisotropic noise. We also realize another work on the anisotropic randomized smoothing \cite{eiras2021ancer} has been established. Our theoretical results are the same with theirs but we follows a different way to establish the theory of anisotropic randomized smoothing. 

%% file: preliminary.tex
We first review the randomized smoothing with isotropic Gaussian noise \cite{cohen2019certified}.

We study the classification from $\mathbb{R}^d$ to classes $\mathcal{Y}$. Given an arbitrary base classifier $f$, randomized smoothing is a method that can turn the base classifier into a ``smoothed'' classifier $g$ by injecting noise into the input. The smoothed classifier predicts the top-1 class w.r.t. to the input $x$ over the noise. The randomized smoothing in Cohen et al. \cite{cohen2019certified} is formally defined as:
\vspace{-0.1in}

\begin{align}\small
\label{eq:cohen's g}
    g(x)=\arg \max_{c\in \mathcal{Y}} \mathbb{P}(f(x+\epsilon)=c), \epsilon \sim \mathcal{N}(0, \sigma^2 \mathbf{I})
\end{align}

\vspace{-0.05in}


The injected noise $\epsilon$ follows an independent isotropic Gaussian noise $\mathcal{N}(0,\sigma^2 \textbf{I})$. Also, the mean of the Gaussian noise is set to be $0$. Thus, the randomized smoothing proposed by \cite{cohen2019certified} adds noise to all the dimensions of the inputs with identical variance and zero-mean.

Based on the smoothed classifier $g$, the top-1 class is denoted as $c_A \in \mathcal{Y}$, the second probable class is denoted as $c_B\in \mathcal{Y}$, and the corresponding lower bound and upper bound of the class probabilities are denoted as $\underline{p_A}$ and $\overline{p_B}$. Cohen et al. \cite{cohen2019certified} derives the first tight bound of the certified radius with the isotropic Gaussian noise in Theorem \ref{thm:cohen's thm}.

\begin{theorem}[\textbf{Randomized Smoothing with Isotropic Gaussian Noise \cite{cohen2019certified}}]
\label{thm:cohen's thm}
Let $f \ : \ \mathbb{R}^d \to \mathcal{Y}$ be any deterministic or random function, and let $\epsilon \sim \mathcal{N}(0,\sigma^2 \mathbf{I})$. Define $g$ as in Eq. (\ref{eq:cohen's g}). For a specific $x\in \mathbb{R}^d$, there exist $c_A\in \mathcal{Y}$ and $\underline{p_A}$, $\overline{p_B} \in [0,1]$ such that:

\vspace{-0.1in}

\begin{equation}\small
\label{eq:cohen's condition}
    \mathbb{P}(f(x+\epsilon)=c_A) \geq \underline{p_A} \geq \overline{p_B} \geq \max_{c \neq c_A} \mathbb{P}(f(x+\epsilon)=c)
\end{equation}

\vspace{-0.05in}

Then $g(x+\delta)=c_A$ for all $||\delta||_2 < R$, where
\begin{equation}\small
\label{eq:cohen's radius}
    R=\frac{\sigma}{2}(\Phi^{-1}(\underline{p_A})-\Phi^{-1}(\overline{p_B}))
\end{equation}
where $\delta$ denotes the perturbation.


\begin{proof}
See detailed proof in \cite{cohen2019certified}. 
\end{proof}
\end{theorem}

Theorem \ref{thm:cohen's thm} guarantees that the smoothed classifier will consistently predict the most probable class when the perturbation is within the radius defined in Eq. (\ref{eq:cohen's radius}) if $p_A$ and $p_B$ satisfy the condition (\ref{eq:cohen's condition}) in Theorem \ref{thm:cohen's thm}.

While certifying any arbitrary classifier, isotropic randomized smoothing applies an identical distribution to generate the noise for all the dimensions, which may limit the defense performance on all the pixels of different inputs. Thus, it is desirable to extend the randomized smoothing to add heterogeneous noise for different pixels (anisotropic). 

However, there are two challenges on extending the isotropic RS to anisotropic RS. First, anisotropic RS needs more complicated theories on deriving the certified radius since the noise follows different distributions on different dimensions. Second, instead of simply adding noise with the same distribution to all the pixels, we need to fine-tune the noise distribution for each pixel in the anisotropic RS. To address these challenges, we propose a novel theory on anisotropic RS in Section \ref{sec: ARS} and design a novel mechanism for finding the optimal noise distribution for all the pixels in Section \ref{sec: NoiseGenerator}.

%% file: ARS.tex
In this section, we propose the first anisotropic randomized smoothing (ARS) theory with the tight certified radius. We take the Gaussian noise as an example to illustrate how to extend the isotropic randomized smoothing to anisotropic randomized smoothing, while other noise can be also readily extended for ARS using similar procedures. Specifically, we also theoretically derive the certified radius for anisotropic Laplace noise in Section \ref{sec: Discussion} as an extension.

For the Gaussian noise, we first extend the smoothed classifier in Eq. (\ref{eq:our g}), and then provide a tight guarantee on its robustness with the anisotropic noise in Theorem \ref{thm:our thm}.

\begin{equation}\small
\label{eq:our g}
    g'(x)=\arg \max_{c\in \mathcal{Y}} \mathbb{P}(f(x+\epsilon)=c), \ \epsilon \sim \mathcal{N}(\mathbf{\mu},\mathbf{\Sigma})
\end{equation}
where the mean of the Gaussian noise is defined as $\mathbf{\mu}=[\mu_1, \mu_2, ..., \mu_d]$, and the variance of the Gaussian noise is defined as $\mathbf{\Sigma}=diag(\sigma_1^2, \sigma_2^2, ..., \sigma_d^2)$.

\begin{restatable}[\textbf{Randomized Smoothing with Anisotropic Gaussian Noise}]{theorem}{thmone}
\label{thm:our thm}
Let $f \ : \ \mathbb{R}^d \to \mathcal{Y}$ be any deterministic or random function, and let $\epsilon \sim \mathcal{N}(\mathbf{\mu},\mathbf{\Sigma})$. Let $g'$ be defined as in Eq. (\ref{eq:our g}). Suppose that for a specific $x\in \mathbb{R}^d$, there exist $c_A\in \mathcal{Y}$ and $\underline{p_A}$, $\overline{p_B} \in [0,1]$ such that:

\begin{equation}\small
\label{eq:condition}
    \mathbb{P}(f(x+\epsilon)=c_A) \geq \underline{p_A} \geq \overline{p_B} \geq \max_{c \neq c_A} \mathbb{P}(f(x+\epsilon)=c)
\end{equation}

Then $g'(x+\delta)=c_A$ for all $||\delta||_2 < R$, where
\begin{equation}\small
\label{eq:our radius}
    R=\frac{1}{2}\min \{\sigma_i\} (\Phi^{-1}(\underline{p_A})-\Phi^{-1}(\overline{p_B})) 
\end{equation}
where $\sigma_i$ denotes the variance on $i$-th dimension of the input, $\delta$ denotes the perturbation.
\end{restatable}

\begin{proof}
See detailed proof in Appendix \ref{apd: proof thm 4.1}. 
\end{proof}

Indeed, Theorem \ref{thm:our thm} can be considered as a generalized form of Theorem \ref{thm:cohen's thm} since when $\mathbf{\Sigma}=diag(\sigma^2,\sigma^2,...,\sigma^2)$ and $\mathbf{\mu}=\mathbf{0}$, we have $\min\{\sigma_i\}=\sigma$. Thus, our theorem returns the same certified radius as Theorem \ref{thm:cohen's thm} in the same setting.

In Theorem \ref{thm:our thm}, we observe that the certified radius only depends on the minimum variance over all the dimensions. Thus, any larger variance in other dimensions would not affect the certified radius. We will show that this provides benefits on defending against the attacks to the randomized smoothing in Section \ref{sec: Experiments}. When we certify the perturbed inputs, the large noise will also smooth the perturbation so that the adversarial effects can be reduced.

We also observe that the certified radius does not depend on the mean of the Gaussian noise. However, a proper mean of the noise may affect the smoothed classifier's prediction on the clean input, and further improve the certified radius since it affects the $p_A$ and $p_B$. We will show how to design a mechanism to find a proper mean (besides the variances) for the noise to improve the certified robustness.

We also present the binary case of Theorem \ref{thm:our thm} as below:

\begin{restatable}[\textbf{Binary Case}]{theorem}{thmtwo}
\label{thm:our thm binary}
Let $f \ : \ \mathbb{R}^d \to \mathcal{Y}$ be any deterministic or random function, and let $\epsilon \sim \mathcal{N}(\mathbf{\mu},\mathbf{\Sigma})$. Let $g'$ be defined as in Eq. (\ref{eq:our g}). Suppose that for a specific $x\in \mathbb{R}^d$, there exist $c_A\in \mathcal{Y}$ and $\underline{p_A}$ such that:
\begin{equation}
    \mathbb{P}(f(x+\epsilon)=c_A) \geq \underline{p_A} \geq \frac{1}{2}
\end{equation}
Then $g'(x+\delta)=c_A$ for all $||\delta||_2 < R$, where
\begin{equation}
\label{eq:binary's radius}
    R=\min \{\sigma_i\} \Phi^{-1}(\underline{p_A})
\end{equation}
\end{restatable}

\begin{proof}
See detailed proof in Appendix \ref{apd: proof thm 4.2}. 
\end{proof}

%% file: NoiseGenerator.tex
\begin{figure*}[!h]
    \centering
    \includegraphics[width=160mm]{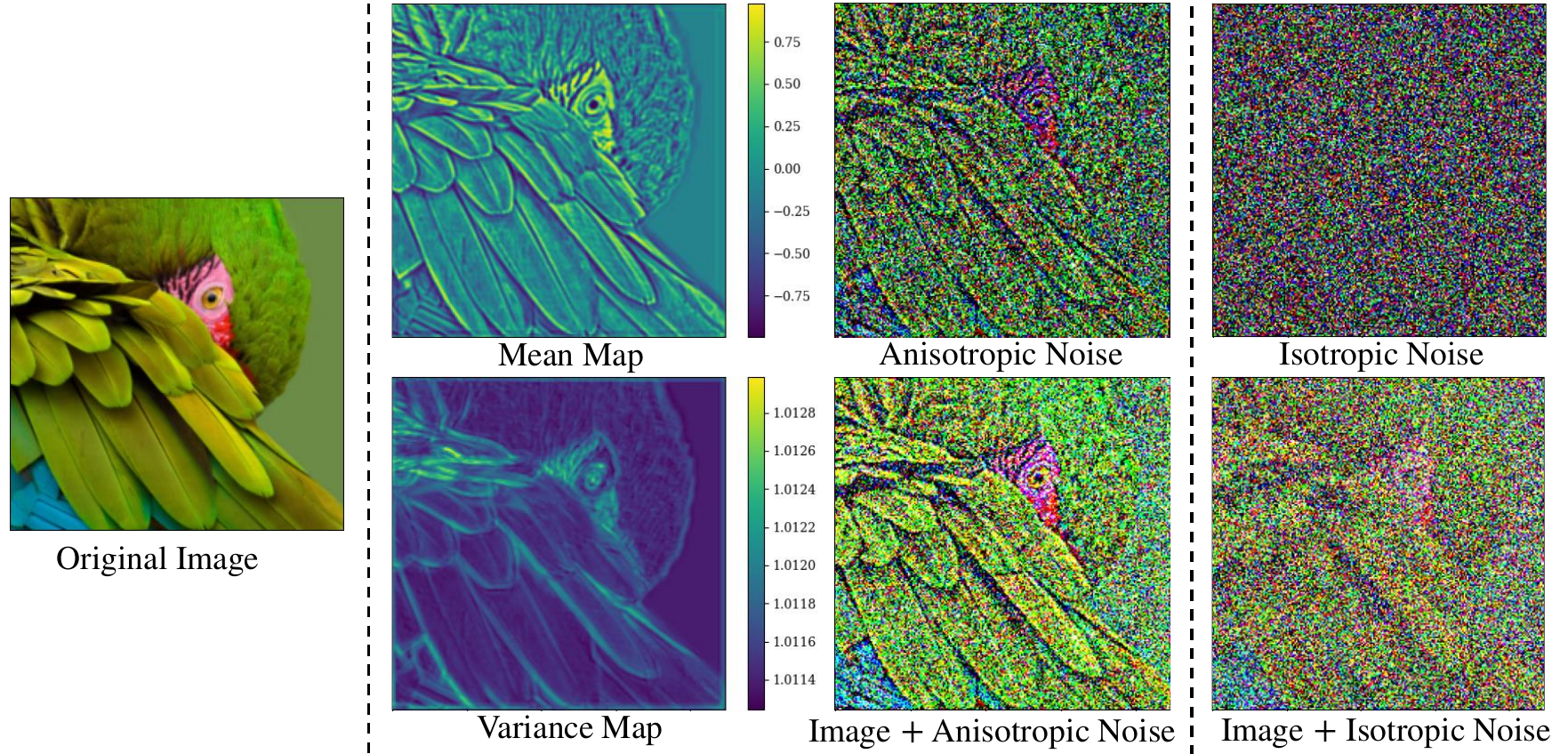}
    \caption{\textmd{An example of anisotropic and isotropic noise. \textbf{Left}: The original image. \textbf{Middle}: The pixel-wise means and variances for anisotropic Gaussian distribution generated by our Noise Generator, and the noise sample. \textbf{Right}: The noise sample generated with isotropic Gaussian distribution of $\sigma=1.0$.}}
    \label{fig:visualization}
\end{figure*}

Our proposed ARS theory could certify the defenses for randomized smoothing based on applying different means and different variances to generate noise for different pixels. However, how to design a new mechanism to fine-tune variance and mean for the noise distributions is a challenging problem. Note that the optimal variances and means can be different from input to input and from pixel to pixel. Therefore, we leverage the CNNs to design a Noise Generator for learning the mapping from the input to the optimal variances and means, and generating the input-dependent variances and means for the certification (see Figure \ref{fig:visualization} for an example of comparing the anisotropic and isotropic noises).

Specifically, in the training and the certification, the Noise Generator takes the image as input and returns a variance map as well as a mean map for the randomized smoothing. Then, the base classifier will take the noisy images as the input for training or classification. The framework is summarized in Figure \ref{fig:framework}. 

\begin{figure}[!h]
    \centering
    \includegraphics[width=80mm]{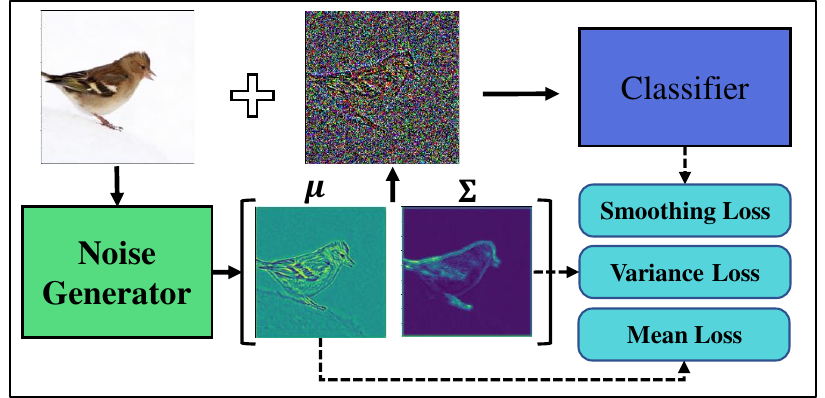}
    \caption{\textmd{\textbf{Framework}. The noise generated by Noise Generator will be added to the image for smoothed classifier training and classification. We train the Noise Generator and the classifier simultaneously with three losses.}}
    \label{fig:framework}
\end{figure}

\textbf{Architecture}.  The Noise Generator learns the mapping from the image to the variance and mean maps, which is similar to the function of the neural networks in image transformation. Therefore, inspired by the image super-resolution work \cite{zhang2018residual}, we also use the ``dense blocks'' \cite{huang2017densely} as the main architecture. It consists of 4 convolution
layers followed by leaky-ReLU \cite{xu2015empirical}. In addition, some special designs are integrated into the Noise Generator to fit our tasks (See Figure \ref{fig:noisegenerator}). First, the output of the dense block is separated into two branches to generate the mean map and variance map. Second, a sigmoid layer is inserted before the mean and variance maps to constrain the mean and variance values. Otherwise, the training may not converge due to some extreme values in the noise. Note that, our Noise Generator is a small network (5-layer deep), so it can be plugged before any classifier for customizing the noise distribution without consuming too much computing resources (See Section \ref{sec: Discussion} for detailed discussion on running time). 

\begin{figure}[!h]
    \centering
    \includegraphics[width=80mm]{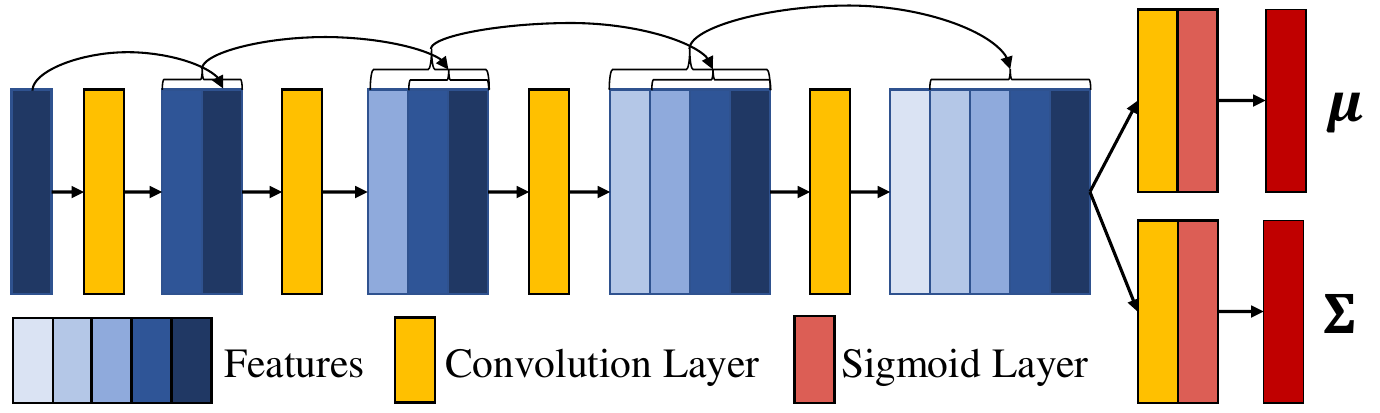}
    \caption{\textmd{\textbf{Architecture of Noise Generator}. }}
    \label{fig:noisegenerator}
\end{figure}

\textbf{Loss Functions}. We train the Noise Generator and the base classifier simultaneously. In the training, our goal is to generate the optimal mean and variance maps. Specifically, by fine-tuning the mean and variance, we aim to train the classifier to predict the noisy image as accurately as possible (large $p_A$). Thus, we use the smoothing loss (Eq. \ref{eq:smoothing loss}) for training both the Noise Generator and the classifier.

\begin{equation}
\label{eq:smoothing loss}
    \mathcal{L}_{s}=-\sum_{k=1}^N y_k \  log[\hat{y}_k(x+\epsilon, \theta_f, \theta_g)]
\end{equation}

where $y_k=1$ if the class $k$ is the correct label of input $x$, otherwise $y_k=0$. $\hat{y}_k$ denotes the prediction of the base classifier $f$ on the input $x$ perturbed by noise $\epsilon$. $\theta_f$ and $\theta_g$ denote the model parameters of classifier $f$ and Noise Generator, respectively.

Also, since the certified radius only depends on the minimum variance, we only have constraints on the minimum values in the variance map. Large variances can improve the certified radius while they will degrade the prediction accuracy since large noises may greatly distort the image. It has been widely known that the variance tunes the trade-off between the accuracy and the certified radius \cite{cohen2019certified, yang2020randomized}. In our case, it is the minimum variance that tunes the trade-off. Therefore, similar to recent works \cite{cohen2019certified}, we constrain the minimum variance to a certain level by the variance loss (Eq. \ref{eq:variance loss}). 

\begin{equation}
\label{eq:variance loss}
    \mathcal{L}_{v}=\left|\frac{\min \{\sigma_i(x, \theta_g)\}-\sigma_0}{\sigma_0}\right|
\end{equation}

where the $\sigma_i(x)$ is the variance for dimension $i$ of the input $x$. $\sigma_0$ denotes the variance target that the minimum variance is trained to achieve. By minimizing the variance loss, we aim to train the Noise Generator to generate a variance map with the minimum value $\min\{\sigma\}_i=\sigma_0$. 

We also constrain the mean map generation with the mean loss (Eq. (\ref{eq:mean loss})) by considering this: although the mean of the noise will not affect the certified radius and can help align the input to the representation center of its class, an extremely large value of mean can distort the image which would harm the prediction accuracy. Thus, the mean map should be as small as possible.

\begin{equation}
\label{eq:mean loss}
    \mathcal{L}_{m}=||\mu(x,\theta_g)||_2
\end{equation}

The training process is to minimize the total loss in Eq. (\ref{eq:total loss})

\begin{equation}
\label{eq:total loss}
    \min_{\theta_f,\theta_g} \quad \mathbb{E}_{x \sim \mathbb{D},\epsilon \sim \mathcal{N}(\mu,\mathbf{\Sigma})} [\alpha \mathcal{L}_s+\beta\mathcal{L}_v + \gamma\mathcal{L}_m]
\end{equation}

where $\alpha$, $\beta$, and $\gamma$ are the weights of the three loss functions, and $\mathbb{D}$ denotes the dataset.

\textbf{Practical Algorithms}. We follow Cohen et al. \cite{cohen2019certified} to use the Monte Carlo algorithm for evaluating $g(x)$ and compute the certified robustness. Different from Cohen et al. \cite{cohen2019certified}, our noise distributions are generated by Noise Generator for each input. Our algorithms for certification and prediction in binary case are presented in Algorithm \ref{alg:certify} and \ref{alg:predict} in Appendix \ref{sec:algorithm}, respectively.

Specifically, in the certification (Algorithm \ref{alg:certify}), the mean and variance for the noise are generated by Noise Generator. Then, we select the top-$1$ class $\hat{c_A}$ by the $ClassifySamples$ function, in which the base classifier outputs the predictions on the noisy input sampled from the noise distribution. Once the top-$1$ class is determined, classification will be run on more samples and the $LowerConfBound$ function will output the lower bound of the probability $\underline{p_A}$ computed by the Binomial test. If $\underline{p_A} > \frac{1}{2}$, we output the prediction class and the certified radius. Otherwise, it outputs ABSTAIN. In the prediction (Algorithm \ref{alg:predict}), we also generate the noise distribution and then compute the prediction counts over the noisy inputs. If the Binomial test succeeds, then it outputs the prediction class. Otherwise, it returns ABSTAIN.

%% file: Experiments.tex
We evaluate the performance of certified robustness in this section. In addition, we evaluate the enhanced robustness of the randomized smoothing with anisotropic noises.

\textbf{Metrics}. Following \cite{cohen2019certified}, we use the \emph{approximate certified test set accuracy} to measure the certified robustness, which is defined as the fraction of the test set that is certified to be consistently correct within the certified radius $R$. See Eq. (\ref{eq:acc}) for the formal definition.
\vspace{-0.05 in}
\begin{equation}\small
\vspace{-0.05 in}
\label{eq:acc}
    Acc(R)=\frac{\sum_{j=1}^N \mathbf{1}_{[g'(x^j+\delta)=y^j]}}{N} \quad \text{for all} \  ||\delta||_2 \le R
\end{equation}

where $x^j$ and $y^j$ denote the $j$-th sample and its label in the test set. $N$ denotes the number of images in the test set.

\textbf{Experimental Settings}. We evaluate our method on the CIFAR10 \cite{krizhevsky2009learning} and ImageNet datasets \cite{ILSVRC15}. We use the original size of the images in CIFAR10, i.e., $3\times 32 \times 32$, while for ImageNet, we resize the images to $3\times 224 \times 224$. In the training, we train the base classifier and the Noise Generator with all the training set in CIFAR10 and ImageNet. We use the ResNet110 and ResNet50 \cite{he2016deep} as the base classifier for CIFAR10 and ImageNet, respectively. The training loss is computed over $5$ samples from the noise distribution for each image. For the certification, following \cite{cohen2019certified}, we evaluate the certified accuracy on the entire test set in CIFAR10 while randomly sample $500$ samples in the test set of ImageNet. In the certification, we also follow \cite{cohen2019certified} to set $\alpha=0.001$ and numbers of Monte Carlo samples $n_0=100$ and $n=100,000$. 

\textbf{Experimental Environment}. All the experiments were performed on the NSF Chameleon Cluster \cite{keahey2020lessons} with Intel(R) Xeon(R) Gold 6230 CPU @ 2.10GHz, 128G RAM, and Tesla V100 SXM2 32GB.

\subsection{Certified Accuracy}
\label{exp:certified accuracy}
We evaluate the certified accuracy on both CIFAR10 and ImageNet, and compare our ARS with the isotropic RS baseline \cite{cohen2019certified}. In \cite{cohen2019certified}, we present the certified accuracy computed with Gaussian noise. Following the setting in \cite{cohen2019certified}, the noise variance is set as $\sigma=0.12$, $0.25$, $0.5$, and $1.0$ in CIFAR10 and $\sigma=0.25$, $0.5$, $1.0$ in ImageNet. In our method, we also set the variance target $\sigma_0$ to be consistent with \cite{cohen2019certified}.

We show the certified accuracy for different certified radii on CIFAR10 and ImageNet datasets in Fig. \ref{fig:cifar10} and Fig. \ref{fig:imagenet}, respectively. We can observe that our certified accuracy is higher than the baseline's in the case of all the certified radii. In particular, when the variance is large, we observe a significant improvement in the certified accuracy using our method. This might be because our mean map in anisotropic Gaussian noise can bring the data representation to the center of the correct class and further improve the prediction probability on the noisy inputs (tuned by noise variance). 

\begin{figure}[t]
    \centering
    \vspace{-0.1in}
    \includegraphics[width=85mm]{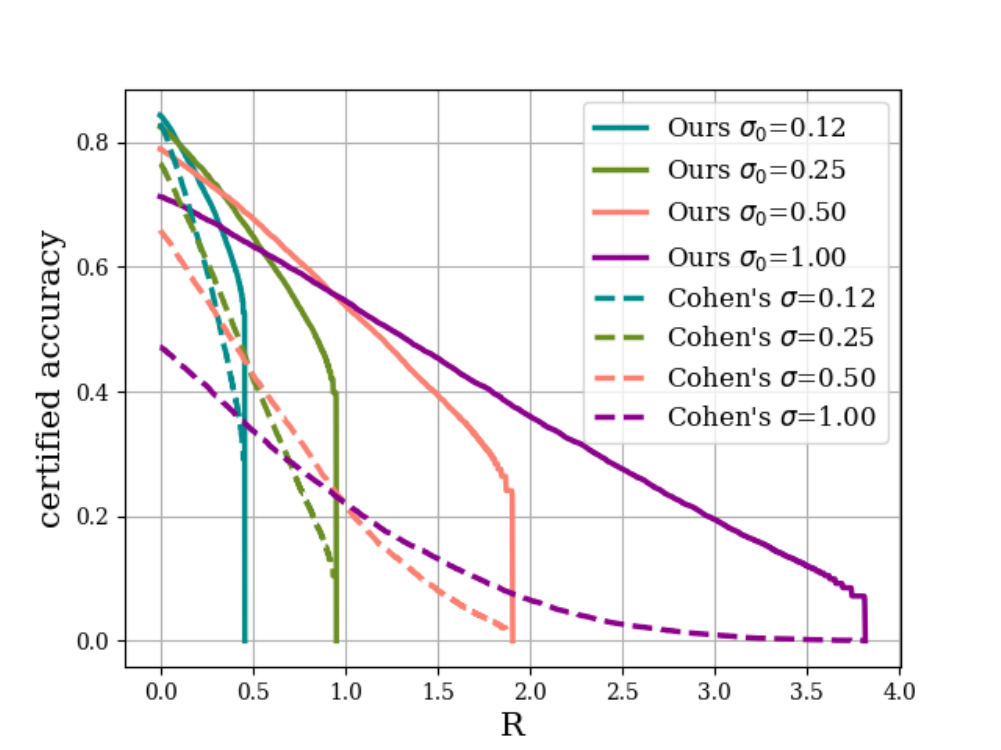}
        \vspace{-0.3 in}
    \caption{\textmd{\textbf{Certified accuracy comparison on CIFAR10 }.}}\vspace{-0.2in}
    \label{fig:cifar10}
\end{figure}

\begin{figure}[t]
    \centering
    \vspace{-0.1in}
    \includegraphics[width=85mm]{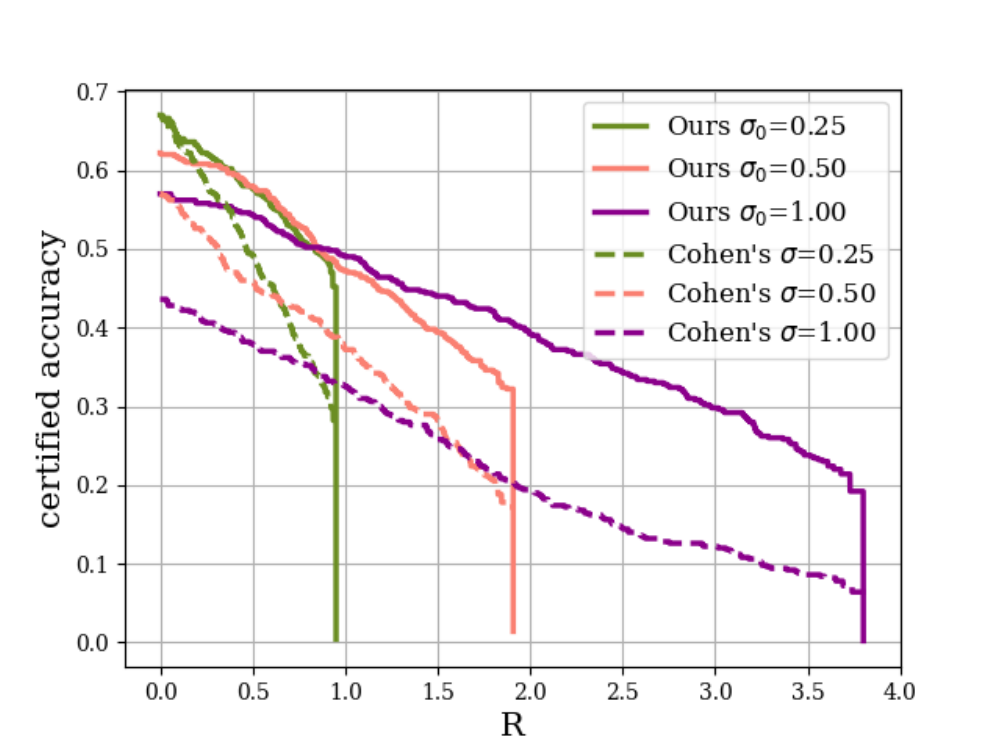}
    \vspace{-0.3 in}
    \caption{\textmd{\textbf{Certified accuracy comparison on ImageNet }.}}\vspace{-0.2in}
    \label{fig:imagenet}
\end{figure}

\subsection{Best Performance Comparison}

We compare our anisotropic randomized smoothing with the state-of-the-art randomized smoothing methods against $\ell_2$ perturbations. Specifically, \cite{cohen2019certified} derives the first tight certified radius for Gaussian noise, which is shown to have the best performance of certified robustness over a wide range of distributions \cite{yang2020randomized}. Also, \cite{zhang2020black} proposes an optimization-based randomized smoothing method with a special-designed distribution that outperforms the Gaussian noise. \cite{alfarra2020data} proposes to optimize the variance of the noise distribution for each input to provide a data-dependent randomized smoothing. We compare our method with these state-of-the-art methods in Table \ref{tab: comparison CIFAR10} and \ref{tab: comparison ImageNet}.

Both on the CIFAR10 and the ImageNet, our method significantly improves the certified accuracy. For instance, we observe the best improvement of $29\%$ at $R=1.5$ on CIFAR10, and $9\%$ at $R=2.5$ on ImageNet. Different from the isotropic methods, when the certified radius is large, our method can still provide certified protection for the images. 

\begin{table*}[!h]
    \centering
     \resizebox{\textwidth}{!}{\begin{tabular}{ c  c c c c c c c c c c c c c c c}
    \hline
    \hline
    Radius & 0.0 & 0.25 & 0.50 & 0.75 &1.00 & 1.25 &1.50 & 1.75 & 2.00 & 2.25 & 2.50 & 2.75 & 3.00 & 3.25 & 3.50\\
    \hline
        Cohen's & 83\% & 61\% & 43\% & 32\% & 22\% & 17\% & 14\% & 9\% & 7\% & 4\% & 3\% & 2\% & 1\% & 0 & 0 \\
    Zhang's  & -- & 61\% & 46\% & 37\% & 25\% & 19\% & 16\% & 14\% &11\% & 9\% & -- & -- & -- & -- & --  \\
      Alfarra's  & 82\% & 68\% & 53\% & 44\% & 32\% & 21\% & 14\% & 8\% & 4\% & 1\% & -- & -- & -- & -- & --\\
        Ours &\textbf{84\%} &\textbf{75\%} & \textbf{68\%} & \textbf{61\%} & \textbf{55\%} & \textbf{50\%} & \textbf{45\%} & \textbf{41\%} & \textbf{36\%} & \textbf{32\%} & \textbf{28\%} & \textbf{23\%} & \textbf{19\%} & \textbf{16\%} & \textbf{12\%}\\
    \hline
    \hline
    \end{tabular}}
    \caption{Certified Accuracy ($\%$) on CIFAR10.}
    \label{tab: comparison CIFAR10}
\end{table*}

\begin{table}[!h]
    \centering
    \resizebox{\columnwidth}{!}{%
    \begin{tabular}{ cccc c c c c c }
    \hline
    \hline
    Radius    & 0.0 & 0.5 & 1.0 & 1.5 & 2.0 & 2.5 & 3.0 & 3.5\\
    \hline
        Cohen's & \textbf{67\%} & 49\% & 37\% & 28\% & 19\% &15\% & 12\% & 9\% \\
        Zhang's & -- & 50\% & 39\% & 31\% & 21\% & 17\% & 13\% & 10\% \\
        Alfarra's & 62\% & \textbf{59\%} & 48\% &43\% &31\% &25\% &22\% &19\%  \\
        Ours   & \textbf{67\%} & 58\% & \textbf{49\%} & \textbf{44\%} & \textbf{39\%} & \textbf{34\%} & \textbf{30\%} & \textbf{24\%}\\
    \hline
    \hline
    \end{tabular}
    }
    \caption{Certified Accuracy ($\%$) on ImageNet.}
    \label{tab: comparison ImageNet}
\end{table}

\subsection{Enhanced Robustness against Pre-Perturbation}


\begin{table}[!h]
    \centering
     \resizebox{\columnwidth}{!}{%
    \begin{tabular}{ cc c c c c c c c c}
    \hline
    \hline
    Radius    & 0.0 & 0.5 & 1.0 & 1.5 & 2.0 & 2.5 & 3.0 & 3.5\\
    \hline
        Cohen's & 44\% & 38\% &33\% & 26\% &19\% &15\% &12\% &9\% \\
        PGD & 30\% & 24\% &19\% &14\% &10\% &7\% &6\% & 5\%  \\
        loss ($\%$) & -31\% & \textbf{-37\%} & -42\% & -46\% & -47\% &-53\% & -50\% & -44\% \\
    \hline
        Ours & 57\% & 54\% & 49\% & 44\% & 39\% &34\% &30\% &24\% \\
        PGD & 40\% & 33\% & 30\% & 27\% & 23\% &20\% &17\% &14\% \\
        loss ($\%$) & \textbf{-30\%} & -39\% & \textbf{-39\%} &\textbf{-39\%} & \textbf{-41\%} & \textbf{-41\%} & \textbf{-43\%} &\textbf{-42\%} \\
    \hline
    \hline
    \end{tabular}
    }
    \caption{Certified Accuracy ($\%$) before and after the PGD attack (pre-perturbing the inputs before certification).}\vspace{-0.1in}
    \label{tab:robust of RM}
\end{table}


We also study a new interesting problem for randomized smoothing methods which has not been discussed in existing works \cite{cohen2019certified, zhang2020black,alfarra2020data}. 
In general, the randomized smoothing is applied to certify the clean images such that the prediction can be guaranteed to be correct if the clean image is perturbed within the certified radius. If the adversary crafts an adversarial example for the certification, the guaranteed prediction could be consistently wrong within the certified radius (consistent with the class label for the adversarial example), which contrarily enhances the attack performance instead of protecting the inputs. Therefore, it is important to build a ``shield'' for the randomized smoothing methods. Different from isotropic randomized smoothing methods, e.g., \cite{cohen2019certified}, our generated mean map can re-calibrate the data point towards its correct label, which makes it harder to be attacked. In addition, our anisotropic randomized smoothing can leverage larger and more complicated noises to mitigate malicious perturbations.

We evaluate the performance of our method on defending against the strong white-box attacks, e.g., PGD attack \cite{madry2018towards} (pre-perturbing). Specifically, the image is perturbed by the PGD attack with max $\ell_\infty$ perturbation $16/255$ and $10$ iterations before the certification, then we compute the certified accuracy on the clean image and the adversarial image and present the loss of the certified accuracy. The variance $\sigma$ and the target variance $\sigma_0$ are both set to $1.0$. Other experimental settings are the same as Section \ref{exp:certified accuracy}.

The experimental results are shown in Table \ref{tab:robust of RM}. Although the certified accuracy for both our method and \cite{cohen2019certified} are degraded by the pre-perturbing PGD attacks, our anisotropic randomized smoothing is still more robust (less degradation) than \cite{cohen2019certified} against such attacks in almost all the cases.

%% file: discussion.tex
\subsection{Generalization to Other Noise Distributions against Different $\ell_p$ Perturbations}

In this paper, we take the anisotropic Gaussian noise as an example for deriving the certified radii and designing the Noise Generator. In fact, anisotropic randomized smoothing and Noise Generator are general methods that can be used for different distributions against different perturbations. Here we show an extension of our theory to anisotropic Laplace noise against $\ell_1$ perturbations in Theorem \ref{thm:our Laplace}. 

\begin{restatable}[\textbf{Randomized Smoothing with Anisotropic Laplace Noise}]{theorem}{thmthree}
\label{thm:our Laplace}
Let $f : \mathbb{R}^d \to \mathcal{Y}$ be any deterministic or random function, and let $\epsilon \sim \mathcal{L}(\mu,\mathbf{\Lambda})$, where $\Lambda=diag(\lambda_1, \lambda_2, ...,\lambda_d)$. Let $g''$ be defined as $ g''(x)=\arg \max_{c\in \mathcal{Y}} \mathbb{P}(f(x+\epsilon)=c)$. Suppose that for a specific $x\in \mathbb{R}^d$, there exist $c_A\in \mathcal{Y}$ and $\underline{p_A}$, $\overline{p_B} \in [0,1]$ such that:
\begin{equation}\small
    \mathbb{P}(f(x+\epsilon)=c_A) \geq \underline{p_A} \geq \overline{p_B} \geq \max_{c \neq c_A} \mathbb{P}(f(x+\epsilon)=c)
\end{equation}
Then $g''(x+\delta)=c_A$ for all $||\delta||_1 < R$, where
\begin{equation}\small
\begin{aligned}
    R= \max & \{\frac{1}{2} \min \{\lambda_i\} \log (\underline{p_A}/\overline{p_B}), \\
    &-\min \{\lambda_i\} \log(1-\underline{p_A}+\overline{p_B})\}
\end{aligned}
\end{equation}
where $\lambda_i$ is the variance on $i$-th dimension of the input.

\end{restatable}

\begin{proof}

See detailed proof in Appendix \ref{apd: proof thm 7.1}.
\end{proof}

Similarly, for other noise distributions, e.g., Exponential, Uniform, or Pareto distributions \cite{yang2020randomized}, we can derive the certified radius with anisotropic noise using the same method. In addition, once the certified radius is derived, our Noise Generator can be used for generating the parameters for any anisotropic noise distribution.

\subsection{Runtime}

Our anisotropic randomized smoothing relies on the Noise Generator to provide optimal protection, which may need extra runtime for generating the mean and variance than traditional RS methods. However, the extra runtime resulting from the Noise Generator is actually negligible. Our model (including Noise Generator) can be trained offline and tested online as traditional RS methods. We evaluate the online certification runtime for our method and \cite{cohen2019certified} on ImageNet with four Tesla V100 GPUs and $2,000$ batch size, the average runtimes over $500$ samples are $27.43$s and $27.09$s per sample for our method and Cohen et al.'s method, respectively. Thus, the Noise Generator will not affect the overall runtime of the certification much.

%% file: Conclusion.tex
We study a new direction of randomized smoothing: the anisotropy of the distributions. Facilitated by the proposed Noise Generator, our anisotropic randomized smoothing significantly improves the certified robustness, which has been extensively evaluated on CIFA10 and ImageNet datasets. The anisotropic randomized smoothing and the Noise Generator can be adapted to various distributions for further improving the certified robustness in the future.

%% file: appendix.tex
\section{Proofs for Theorem \ref{thm:our thm}}
\label{apd: proof thm 4.1}
We prove the Theorem \ref{thm:our thm} in this section. Similar to Cohen et al. \cite{cohen2019certified}, Theorem \ref{thm:our thm} is based on Neyman-Pearson lemma \cite{neyman1933ix}. Therefore, we will review the Neyman-Pearson lemma and then derive the certified radius for anisotropic Gaussian noise.

\begin{lemma}[\textbf{Neyman-Pearson} \cite{neyman1933ix}]
\label{lemma:neyman-pearson}

Let $X$ and $Y$ be random variables in $\mathbb{R}^d$ with probability density functions (PDF) $f_X$ and $f_Y$. Let $h:\mathbb{R}^d \rightarrow \{0,1\}$ be a random or deterministic function. Then:
\begin{itemize}

\item[(1)] If $S=\{z\in \mathbb{R}^d: \frac{f_Y(z)}{f_X(z)}\leq{t}\}$ for some $t>0$ and $\mathbb{P}(h(X)=1)\ge \mathbb{P}(X\in S)$, then $\mathbb{P}(h(Y)=1) \ge \mathbb{P}(Y\in S)$;

\item[(2)] If $S=\{z\in \mathbb{R}^d: \frac{f_Y(z)}{f_X(z)}\ge {t}\}$ for some $t>0$ and $\mathbb{P}(h(X)=1)\leq \mathbb{P}(X\in S)$, then $\mathbb{P}(h(Y)=1) \leq \mathbb{P}(Y\in S)$.
\end{itemize}
\end{lemma}

\begin{proof}
See the detailed proof in Cohen et al. \cite{cohen2019certified}
\end{proof}

Then, we prove the special case of Lemma \ref{lemma:neyman-pearson} when the random variables follows independent anisotropic Gaussian distribution.

\begin{lemma}[]
\label{lemma: anisotropic NP lemma}

Let $X \sim \mathcal{N}(x+\mu,\mathbf{\Sigma})$ and $Y \sim \mathcal{N}(x+\mu+\delta, \mathbf{\Sigma})$, where $\delta=[\delta_1, \delta_2, ..., \delta_d]$, $\mu=[\mu_1, \mu_2, ..., \mu_d]$ and $\mathbf{\Sigma}=diag(\sigma_1^2, \sigma_2^2, ..., \sigma_d^2)$. Let $h: \mathbb{R}^d \to \{0,1\}$ be any deterministic or random function. Then:
\begin{itemize}
    \item[(1)] If $S= \{ z\in \mathbb{R}^d: \sum_{i=1}^d \frac{\delta_i}{\sigma_i^2}z_i \leq \beta \}$ for some $\beta$ and $\mathbb{P}(h(X)=1) \ge \mathbb{P}(X\in S)$, then $\mathbb{P}(h(Y)=1) \ge \mathbb{P}(Y\in S)$

    \item[(2)] If $S= \{ z\in \mathbb{R}^d: \sum_{i=1}^d \frac{\delta_i}{\sigma_i^2}z_i \geq \beta \}$ for some $\beta$ and $\mathbb{P}(h(X)=1) \le \mathbb{P}(X\in S)$, then $\mathbb{P}(h(Y)=1) \le \mathbb{P}(Y\in S)$
\end{itemize}
\end{lemma}

\begin{proof}
Let $X \sim \mathcal{N}(x+\mu,\mathbf{\Sigma})$ and $Y \sim \mathcal{N}(x+\mu+\delta, \mathbf{\Sigma})$. We have the probability density functions $f_X$ and $f_Y$ as:

\begin{equation}
    f_X(z)= k \exp \ \left( -\sum_{i=1}^d \frac{1}{2\sigma_i^2}[z_i-(x_i+\mu_i)]^2 \right) \nonumber
\end{equation}

\begin{equation}
    f_Y(z)=k \exp \ \left(-\sum_{i=1}^d \frac{1}{2\sigma_i^2}[z_i-(x_i+\mu_i+\delta_i)]^2 \right) \nonumber
\end{equation}

where $k$ is a constant. The ratio of the PDF is:

\begin{equation}
\begin{aligned}
    \frac{f_Y(z)}{f_X(z)} & =\frac{\exp \ (-\sum_{i=1}^d \frac{1}{2\sigma_i^2}[z_i-(x_i+\mu_i+\delta_i)]^2)}{\exp \  (-\sum_{i=1}^d \frac{1}{2\sigma_i^2}[z_i-(x_i+\mu_i)]^2)}\\
    & =\exp \ \left(\sum_{i=1}^d \frac{1}{2\sigma_i^2} [2z_i\delta_i-2(x_i+\mu_i)\delta_i-\delta_i^2] \right)\\
    & = \exp \ \left(\sum_{i=1}^d \frac{z_i\delta_i}{\sigma_i^2}-\sum_{i=1}^d \frac{1}{2\sigma_i^2}[2(x_i+\mu_i)\delta_i+\delta_i^2)]\right)\\
    & = \exp \ (\sum_{i=1}^d \frac{\delta_i}{\sigma_i^2}z_i-c)
\end{aligned}
\end{equation}

where $c=\sum_{i=1}^d \frac{1}{2\sigma_i^2}[2(x_i+\mu_i)\delta_i+\delta_i^2)] $, which is constant w.r.t. $z$. Let $\beta=\log t + c$, we have:

\begin{equation}
\begin{aligned}
    \sum_{i=1}^d \frac{\delta_i}{\sigma_i^2}z_i \leq \beta \Longleftrightarrow \exp \ (\sum_{i=1}^d \frac{\delta_i}{\sigma_i^2}z_i-c) \leq t \\
    \sum_{i=1}^d \frac{\delta_i}{\sigma_i^2}z_i \geq \beta \Longleftrightarrow \exp \ (\sum_{i=1}^d \frac{\delta_i}{\sigma_i^2}z_i-c) \geq t \nonumber
\end{aligned}
\end{equation}

Therefore, for any $\beta$, there is some $t>0$ for which:

\begin{equation}
    \{ z: \sum_{i=1}^d \frac{\delta_i}{\sigma_i^2}z_i \leq \beta \} = \{z: \frac{f_Y(z)}{f_X(z)} \le t\} \quad \text{and} \quad  \{ z: \sum_{i=1}^d \frac{\delta_i}{\sigma_i^2}z_i \geq \beta \} = \{z: \frac{f_Y(z)}{f_X(z)} \ge t\}
\end{equation}
This completes the proof.
\end{proof}

Then, we prove the Theorem \ref{thm:our thm}.

\thmone*

\begin{proof}
To prove $g'(x+\delta)=c_A$, it is equivalent to prove that
\begin{equation}
\label{eq:17}
    \mathbb{P}(f(x+\delta+\epsilon)=c_A) > \mathbb{P}(f(x+\delta+\epsilon)=c_B)
\end{equation}
where $c_B$ denotes the second probable class.

For brevity, define the random variables
\begin{equation}
\begin{aligned}
X &:= x+\epsilon = \mathcal{N}(x+\mu,\mathbf{\Sigma})\\
Y &:= x+\delta+\epsilon= \mathcal{N}(x+\mu+\delta,\mathbf{\Sigma})
\end{aligned}    
\end{equation}

Then, proving Eq. (\ref{eq:17}) is equivalent to prove:

\begin{equation}
\label{eq:key eq}
    \mathbb{P}(f(Y)=c_A) > \mathbb{P}(f(Y)=c_B)
\end{equation}

From Eq. (\ref{eq:condition}) we have:

\begin{equation}
\label{eq:20}
    \mathbb{P}(f(X)=c_A) \ge \underline{p_A} \quad \text{and} \quad  \mathbb{P}(f(X)=c_B) \le \overline{p_B}
\end{equation}

Then, we will show how to use Lemma \ref{lemma: anisotropic NP lemma} to prove Eq. (\ref{eq:key eq}) from Eq. (\ref{eq:20}).

First, we define the half-spaces:

\begin{equation}
\label{eq:21}
\begin{aligned}
A & := \{z: \sum_{i=1}^d \frac{\delta_i}{\sigma_i^2} (z_i-\mu_i-x_i) \le \sqrt{\sum_{i-1}^d \frac{\delta_i^2}{\sigma_i^2}} \Phi^{-1}(\underline{p_A})   \} \\
B & := \{z: \sum_{i=1}^d \frac{\delta_i}{\sigma_i^2} (z_i -\mu_i -x_i) \ge \sqrt{\sum_{i-1}^d \frac{\delta_i^2}{\sigma_i^2}} \Phi^{-1}(1-\underline{p_B})  \}
\end{aligned}
\end{equation}

where $\Phi^{-1}$ is the inverse of the standard Gaussian CDF.

Then, we will have:
\begin{equation}
\label{eq:22}
\begin{aligned}
    \mathbb{P}(X \in A) &= \mathbb{P}(\sum_{i=1}^d \frac{\delta_i}{\sigma_i^2} (X_i-\mu_i-x_i) \le \sqrt{\sum_{i-1}^d \frac{\delta_i^2}{\sigma_i^2}} \Phi^{-1}(\underline{p_A}))    \\
    & =\mathbb{P}(\sum_{i=1}^d \frac{\delta_i}{\sigma_i^2} \mathcal{N}(0,\sigma_i^2) \le \sqrt{\sum_{i-1}^d \frac{\delta_i^2}{\sigma_i^2}} \Phi^{-1}(\underline{p_A}))\\
    & =\mathbb{P}(\sum_{i=1}^d \mathcal{N}(0,\frac{\delta_i^2}{\sigma_i^2}) \le \sqrt{\sum_{i-1}^d \frac{\delta_i^2}{\sigma_i^2}} \Phi^{-1}(\underline{p_A}))\\
    & =\mathbb{P}(\sqrt{\sum_{i-1}^d \frac{\delta_i^2}{\sigma_i^2}} \  \mathcal{N}(0,1) \le \sqrt{\sum_{i-1}^d \frac{\delta_i^2}{\sigma_i^2}} \Phi^{-1}(\underline{p_A}))\\
    & = \mathbb{P} (\mathcal{N}(0,1) \le \Phi^{-1}(\underline{p_A}))\\
    & = \underline{p_A}
\end{aligned}
\end{equation}

\begin{equation}
\label{eq:23}
\begin{aligned}
    \mathbb{P}(X \in B) &= \mathbb{P}(\sum_{i=1}^d \frac{\delta_i}{\sigma_i^2} (X_i-\mu_i-x_i) \ge \sqrt{\sum_{i-1}^d \frac{\delta_i^2}{\sigma_i^2}} \Phi^{-1}(1-\overline{p_B}))    \\
    & =\mathbb{P}(\sum_{i=1}^d \frac{\delta_i}{\sigma_i^2} \mathcal{N}(0,\sigma_i^2) \ge \sqrt{\sum_{i-1}^d \frac{\delta_i^2}{\sigma_i^2}} \Phi^{-1}(1-\overline{p_B}))\\
    & =\mathbb{P}(\sum_{i=1}^d \mathcal{N}(0,\frac{\delta_i^2}{\sigma_i^2}) \ge \sqrt{\sum_{i-1}^d \frac{\delta_i^2}{\sigma_i^2}} \Phi^{-1}(1-\overline{p_B}))\\
    & =\mathbb{P}(\sqrt{\sum_{i-1}^d \frac{\delta_i^2}{\sigma_i^2}} \  \mathcal{N}(0,1) \ge \sqrt{\sum_{i-1}^d \frac{\delta_i^2}{\sigma_i^2}} \Phi^{-1}(1-\overline{p_B}))\\
    & = \mathbb{P} (\mathcal{N}(0,1) \ge \Phi^{-1}(1-\overline{p_B}))\\
    & = \overline{p_B}
\end{aligned}
\end{equation}

Now we have $\mathbb{P}(X \in A)=\underline{p_A}$ (Eq. (\ref{eq:22})), so by Eq. (\ref{eq:20}) we have $\mathbb{P}(f(X) = c_A) \ge \mathbb{P}(X\in A)$. Using Neyman-Pearson Lemma with $h(z):=\textbf{1}[f(z)=c_A]$, we have:

\begin{equation}
    \mathbb{P}(f(Y)=c_A) \ge \mathbb{P}(Y\in A)
\end{equation}

Similarly, by Eq. (\ref{eq:20}), Eq. (\ref{eq:23}) and Neyman-Pearson Lemma, we also have:

\begin{equation}
    \mathbb{P}(f(Y)=c_B) \le \mathbb{P}(Y\in B)
\end{equation}

Finally, to prove that $\mathbb{P}(f(Y)=c_A) \ge \mathbb{P}(f(Y)=c_B)$, we will need to prove:
\begin{equation}
    \mathbb{P}(f(Y)=c_A) \ge \mathbb{P}(Y\in A) \ge \mathbb{P}(Y\in B) \ge \mathbb{P}(f(Y)=c_B)
\end{equation}

Recall that $Y \sim \mathcal{N}(x+\mu+\delta,\mathbf{\Sigma})$, and $A := \{z: \sum_{i=1}^d \frac{\delta_i}{\sigma_i^2} (z_i-\mu_i-x_i) \le \sqrt{\sum_{i-1}^d \frac{\delta_i^2}{\sigma_i^2}} \Phi^{-1}(\underline{p_A}) \}$, we can have:

\begin{equation}
\label{eq:27}
\begin{aligned}
\mathbb{P}(Y\in A) & =\mathbb{P}(\sum_{i=1}^d \frac{\delta_i}{\sigma_i^2} (Y_i-\mu_i-x_i) \le \sqrt{\sum_{i-1}^d \frac{\delta_i^2}{\sigma_i^2}} \Phi^{-1}(\underline{p_A})) \\
& = \mathbb{P}(\sum_{i=1}^d \frac{\delta_i}{\sigma_i^2} \mathcal{N}(\delta_i, \sigma_i^2) \le \sqrt{\sum_{i-1}^d \frac{\delta_i^2}{\sigma_i^2}} \Phi^{-1}(\underline{p_A})) \\
& = \mathbb{P}(\sum_{i=1}^d \frac{\delta_i}{\sigma_i^2} \mathcal{N}(0, \sigma_i^2) + \sum_{i=1}^d \frac{\delta_i^2}{\sigma_i^2} \le \sqrt{\sum_{i-1}^d \frac{\delta_i^2}{\sigma_i^2}} \Phi^{-1}(\underline{p_A})) \\
& = \mathbb{P}(\sqrt{\sum_{i-1}^d \frac{\delta_i^2}{\sigma_i^2}} \  \mathcal{N}(0,1) + \sum_{i=1}^d \frac{\delta_i^2}{\sigma_i^2} \le \sqrt{\sum_{i-1}^d \frac{\delta_i^2}{\sigma_i^2}} \Phi^{-1}(\underline{p_A})) \\
& = \mathbb{P}(\mathcal{N}(0,1)  \le \Phi^{-1}(\underline{p_A})- \sqrt{ \sum_{i=1}^d \frac{\delta_i^2}{\sigma_i^2}}) \\
&=\Phi(\Phi^{-1}(\underline{p_A})- \sqrt{ \sum_{i=1}^d \frac{\delta_i^2}{\sigma_i^2}})
\end{aligned}
\end{equation}

Similarly, with $B := \{z: \sum_{i=1}^d \frac{\delta_i}{\sigma_i^2} (z_i -\mu_i -x_i) \ge \sqrt{\sum_{i-1}^d \frac{\delta_i^2}{\sigma_i^2}} \Phi^{-1}(1-\underline{p_B})  \}$, we have:

\begin{equation}
\begin{aligned}
\mathbb{P}(Y\in B) & =\mathbb{P}(\sum_{i=1}^d \frac{\delta_i}{\sigma_i^2} (Y_i-\mu_i-x_i) \ge \sqrt{\sum_{i-1}^d \frac{\delta_i^2}{\sigma_i^2}} \Phi^{-1}(1-\overline{p_B})) \\
& = \mathbb{P}(\sum_{i=1}^d \frac{\delta_i}{\sigma_i^2} \mathcal{N}(\delta_i, \sigma_i^2) \ge \sqrt{\sum_{i-1}^d \frac{\delta_i^2}{\sigma_i^2}} \Phi^{-1}(1-\overline{p_B})) \\
& = \mathbb{P}(\sum_{i=1}^d \frac{\delta_i}{\sigma_i^2} \mathcal{N}(0, \sigma_i^2) + \sum_{i=1}^d \frac{\delta_i^2}{\sigma_i^2} \ge \sqrt{\sum_{i-1}^d \frac{\delta_i^2}{\sigma_i^2}} \Phi^{-1}(1-\overline{p_B})) \\
& = \mathbb{P}(\sqrt{\sum_{i-1}^d \frac{\delta_i^2}{\sigma_i^2}} \  \mathcal{N}(0,1) + \sum_{i=1}^d \frac{\delta_i^2}{\sigma_i^2} \ge \sqrt{\sum_{i-1}^d \frac{\delta_i^2}{\sigma_i^2}} \Phi^{-1}(1-\overline{p_B})) \\
& = \mathbb{P}(\mathcal{N}(0,1)  \ge \Phi^{-1}(1-\overline{p_B})- \sqrt{ \sum_{i=1}^d \frac{\delta_i^2}{\sigma_i^2}}) \\
& = \mathbb{P}(\mathcal{N}(0,1)  \le \Phi^{-1}(\overline{p_B})+ \sqrt{ \sum_{i=1}^d \frac{\delta_i^2}{\sigma_i^2}}) \\
&=\Phi(\Phi^{-1}(\overline{p_B}) + \sqrt{ \sum_{i=1}^d \frac{\delta_i^2}{\sigma_i^2}})
\end{aligned}
\end{equation}

Therefore, to ensure $\mathbb{P}(Y\in A) \ge \mathbb{P}(Y \in B)$, we need:

\begin{equation}
\label{eq:29}
    \Phi(\Phi^{-1}(\underline{p_A})- \sqrt{ \sum_{i=1}^d \frac{\delta_i^2}{\sigma_i^2}}) \ge \Phi(\Phi^{-1}(\overline{p_B}) + \sqrt{ \sum_{i=1}^d \frac{\delta_i^2}{\sigma_i^2}}) \quad \Longleftrightarrow \quad \sqrt{ \sum_{i=1}^d \frac{\delta_i^2}{\sigma_i^2}} \le \frac{1}{2}(\Phi^{-1}(\underline{p_A})-\Phi^{-1}(\overline{p_B}))
\end{equation}

Let $\sigma_m$ denote the minimum $\sigma_i$, we have

\begin{equation}
    \sqrt{\sum_{i=1}^d \frac{\delta_i^2}{\sigma_i^2}}  \leq 
    \sqrt{ \sum_{i=1}^d\frac{ \delta_i^2}{\sigma_m^2}} = \frac{||\delta||_2}{\sigma_m} 
\end{equation}

Therefore, if $\frac{||\delta||_2}{\sigma_m}  \leq \frac{1}{2}(\Phi^{-1}(\underline{p_A})-\Phi^{-1}(\overline{p_B}))$, it can ensure Eq. (\ref{eq:29}). 

To conclude, $g'(x+\delta)=c_A$ is ensured if:

\begin{equation}
    ||\delta||_2  \leq \frac{1}{2}\min \{\sigma_i\}(\Phi^{-1}(\underline{p_A})-\Phi^{-1}(\overline{p_B}))
\end{equation}

Figure \ref{fig:deriving} also illustrates the relationship between the certified radius and Eq. (\ref{eq:29}). 

\begin{figure}[!h]
    \centering
    \includegraphics[width=80mm]{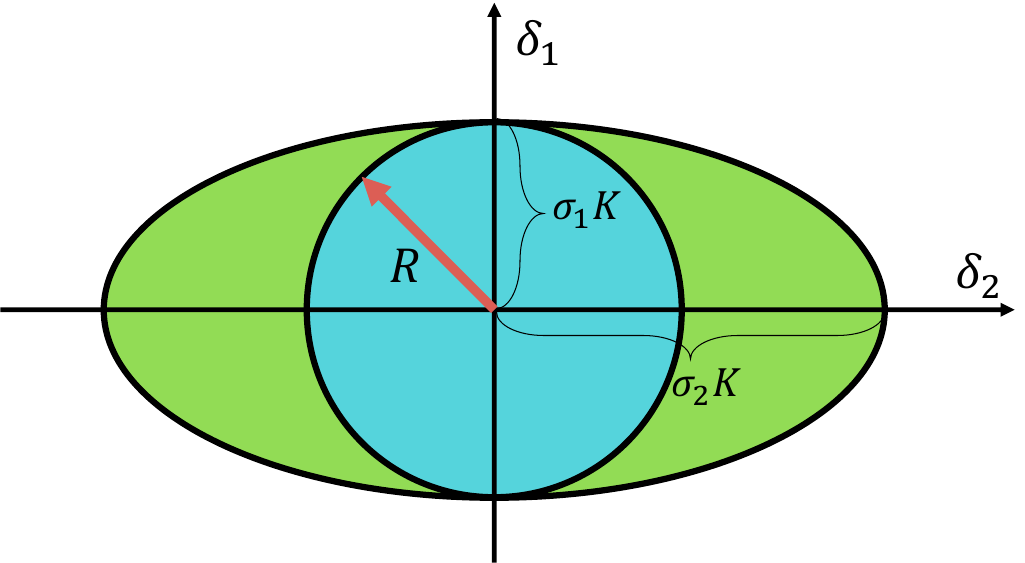}
    \caption{\textmd{\textbf{Illustration}. Considering a $\delta$ space with two dimensions, Eq. (\ref{eq:29}) construct an ellipse with semi-minor axes $\sigma_1 K$ and semi-major axes $\sigma_2 K$, where  $K=\frac{1}{2}(\Phi^{-1}(\underline{p_A})-\Phi^{-1}(\overline{p_B}))$ and $\sigma_1 < \sigma_2$. Within the ellipse, the smoothed classifier's prediction is guaranteed to be $c_A$. To find the certified radius $R$ in $\ell_2$ norm, it is equivalent to find the circle of radius such that the radius is the semi-minor axes of the ellipse since within such circle, the smoothed prediction is constantly to be $c_A$. Therefore, in high-dimensional case, our certified radius is $\min\{\sigma_i\} K$. }
}
    \label{fig:deriving}
\end{figure}

\end{proof}

\section{Proofs for Theorem \ref{thm:our thm binary}}
\label{apd: proof thm 4.2}
Follow the proof in Appendix \ref{apd: proof thm 4.1}, we can prove the binary case of Theorem \ref{thm:our thm}. 

\thmtwo*


\begin{proof}

In the binary case, if $\mathbb{P}(f(x+\delta+\epsilon)=c_A) > 1/2$, we can ensure that $g'(x+\delta)=c_A$, which is also equivalent to prove:

\begin{equation}
    \mathbb{P}(f(Y)=c_A) > \frac{1}{2}
\end{equation}

Define the half-space as in Eq. (\ref{eq:21}), we also have $\mathbb{P}(X\in A)=\underline{p_A}$ by Eq. (\ref{eq:22}). Using Neyman-Pearson Lemma, we have:

\begin{equation}
    \mathbb{P}(f(Y)=c_A) \ge \mathbb{P}(Y\in A)
\end{equation}

To guarantee that $\mathbb{P}(f(Y)=c_A) > \frac{1}{2}$, we need $\mathbb{P}(Y\in A) > \frac{1}{2}$. 

By Eq. (\ref{eq:27}), we have:

\begin{equation}
    \mathbb{P}(Y\in A)=\Phi(\Phi^{-1}(\underline{p_A})- \sqrt{ \sum_{i=1}^d \frac{\delta_i^2}{\sigma_i^2}})
\end{equation}

Therefore, to ensure $\mathbb{P}(f(Y)=c_A) > \frac{1}{2}$, we need:

\begin{equation}
    \Phi(\Phi^{-1}(\underline{p_A})- \sqrt{ \sum_{i=1}^d \frac{\delta_i^2}{\sigma_i^2}}) > \frac{1}{2} \quad \Longleftrightarrow \quad \sqrt{ \sum_{i=1}^d \frac{\delta_i^2}{\sigma_i^2}} < \Phi^{-1}(\underline{p_A})
\end{equation}

Similarly, it is obvious to have:

\begin{equation}
        ||\delta||_2 < \min \{\sigma_i\} \Phi^{-1}(\underline{p_A})
\end{equation}
This completes the proof.
\end{proof}

\section{Algorithms}
\label{sec:algorithm}

\begin{algorithm}[!h]
   \caption{Anisotropic Randomized Smoothing Prediction}
   \label{alg:certify}
\begin{algorithmic}
   \STATE {\bfseries Given:} Base Classifier $f$, Noise Generator $g_n$, Input image $x$, Monte Carlo Sampling Number $n$, confidence $1-\alpha$
   \STATE $\mathbf{\mu}$, $\mathbf{\Sigma} \leftarrow g_n(x)$
   \STATE $counts \leftarrow ClassifySamples(f,x,\mu,\mathbf{\Sigma},n)$
   \STATE $\hat{c}_A, \hat{c}_B \leftarrow \text{\textbf{top two indexes in}} \  counts$
   \STATE $n_A, n_B \leftarrow counts[\hat{c}_A], counts[\hat{c}_B]$
    \IF{$BinomPValue(n_A,n_A+n_B,0.5) \le \alpha$}
        \STATE \textbf{return} prediction $\hat{c}_A$
    \ELSE
        \STATE \textbf{return} ABSTAIN
    \ENDIF
   
\end{algorithmic}
\end{algorithm}

\begin{algorithm}[!h]
   \caption{Anisotropic Randomized Smoothing Certification}
   \label{alg:predict}
\begin{algorithmic}
   \STATE {\bfseries Given:} Base Classifier $f$, Noise Generator $g_n$, Input image $x$, Monte Carlo Sampling Number $n_0$ and $n$, confidence $1-\alpha$
   \STATE $\mathbf{\mu}$, $\mathbf{\Sigma} \leftarrow g_n(x)$
   \STATE $counts\_select \leftarrow ClassifySamples(f,x,\mu,\mathbf{\Sigma},n_0)$
   \STATE $\hat{c}_A \leftarrow \text{\textbf{top index in}} \  counts\_select$
   \STATE $counts \leftarrow ClassifySamples(f,x,\mu,\mathbf{\Sigma},n)$
   \STATE $\underline{p_A} \leftarrow LowerConfBound(counts[\hat{c}_A],n,1-\alpha)$
    \IF{$\underline{p_A}>\frac{1}{2}$}
        \STATE \textbf{return} prediction $\hat{c}_A$ and radius $\min \{\sigma_i\} \Phi^{-1}(\underline{p_A})$
    \ELSE
        \STATE \textbf{return} ABSTAIN
    \ENDIF
   
\end{algorithmic}
\end{algorithm}

\section{Proof for Theorem \ref{thm:our Laplace} }
\label{apd: proof thm 7.1}

Similar to the proof for Theorem \ref{thm:our thm}, we first prove the special case of Lemma \ref{lemma:neyman-pearson} when the random variables follows independent anisotropic Laplace distribution.

\begin{lemma}[]
\label{lemma: anisotropic NP lemma for Laplace}

Let $X \sim \mathcal{L}(x+\mu,\mathbf{\Lambda})$ and $Y \sim \mathcal{L}(x+\mu+\delta, \mathbf{\Lambda})$, where $\delta=[\delta_1, \delta_2, ..., \delta_d]$, $\mu=[\mu_1, \mu_2, ..., \mu_d]$ and $\mathbf{\Lambda}=diag(\lambda_1, \lambda_2, ..., \lambda_d)$. Let $h: \mathbb{R}^d \to {0,1}$ be any deterministic or random function. Then:
\begin{itemize}
    \item[(1)] If $S= \{ z\in \mathbb{R}^d: \sum_{i=1}^d\frac{1}{\lambda_i}(|z_i-\delta_i|-|z_i|) \geq \beta \}$ for some $\beta$ and $\mathbb{P}(h(X)=1) \ge \mathbb{P}(X\in S)$, then $\mathbb{P}(h(Y)=1) \ge \mathbb{P}(Y\in S)$

    \item[(2)] If $S= \{ z\in \mathbb{R}^d: \sum_{i=1}^d\frac{1}{\lambda_i}(|z_i-\delta_i|-|z_i|) \leq \beta \}$ for some $\beta$ and $\mathbb{P}(h(X)=1) \le \mathbb{P}(X\in S)$, then $\mathbb{P}(h(Y)=1) \le \mathbb{P}(Y\in S)$
\end{itemize}
\end{lemma}

\begin{proof}
Let $X \sim \mathcal{L}(x+\mu,\mathbf{\Lambda})$ and $Y \sim \mathcal{L}(x+\mu+\delta, \mathbf{\Lambda})$. We have the probability density functions $f_X$ and $f_Y$ as:

\begin{equation}
    f_X(z)=k \exp \ \left( -\sum_{i=1}^d \frac{1}{\lambda_i}|z_i-(x_i+\mu_i)| \right) \nonumber
\end{equation}

\begin{equation}
    f_Y(z)=k \exp \ \left(-\sum_{i=1}^d \frac{1}{\lambda_i}|z_i-(x_i+\mu_i+\delta_i)| \right) \nonumber
\end{equation}

where $k$ is a constant. The ratio of the PDF is:

\begin{equation}
\begin{aligned}
    \frac{f_Y(z)}{f_X(z)} &=\frac{\exp \ \left( -\sum_{i=1}^d \frac{1}{\lambda_i}|z_i-(x_i+\mu_i)| \right)}{\exp \ \left(-\sum_{i=1}^d \frac{1}{\lambda_i}|z_i-(x_i+\mu_i+\delta_i)| \right)}\\
    & =\frac{\exp \ \left( -\sum_{i=1}^d \frac{1}{\lambda_i}|z_i| \right)}{\exp \ \left(-\sum_{i=1}^d \frac{1}{\lambda_i}|z_i-\delta_i| \right)}\\
    & = \exp \left( -\sum_{i=1}^d\frac{1}{\lambda_i}(|z_i-\delta_i|-|z_i|)      \right)
\end{aligned}
\end{equation}

Let $\beta=-\log t$, we have:

\begin{equation}
\begin{aligned}
    \sum_{i=1}^d\frac{1}{\lambda_i}(|z_i-\delta_i|-|z_i|) \geq \beta \Longleftrightarrow \frac{f_Y(z)}{f_X(z)} \leq t \\
    \sum_{i=1}^d\frac{1}{\lambda_i}(|z_i-\delta_i|-|z_i|) \leq \beta \Longleftrightarrow \frac{f_Y(z)}{f_X(z)} \geq t \\
\end{aligned}
\end{equation}

This completes the proof.
\end{proof}

\thmthree*



\begin{proof}

Denote $T(x) = \sum_{i=1}^d \frac{(|x_i-\delta_i|-|x_i|)}{\lambda_i}$. Use Triangle Inequality for each term in the summation, we can derive a bound for $T(x)$:

\begin{equation}
\label{eq:46}
    -\sum_{i=1}^d \frac{|\delta_i|}{\lambda_i} \leq T(x) \leq \sum_{i=1}^d \frac{|\delta_i|}{\lambda_i}
\end{equation}

Define two sets:

\begin{equation}
\label{eq: two sets for Laplace}
\begin{aligned}
A:= \{ z: T(z) \geq \beta_1\}\\
B:= \{ z: T(z) \leq \beta_2\}
\end{aligned}
\end{equation}
where the $\beta_1$ and $\beta_2$ are selected to suffice:

\begin{equation}
\begin{aligned}
\mathbb{P}(X \in A) = \underline{p_A}\\
\mathbb{P}(X \in B) = \overline{p_B}
\end{aligned}
\end{equation}

Applying Lemma \ref{lemma: anisotropic NP lemma for Laplace} to Eq. (\ref{eq: two sets for Laplace}), we have:

\begin{equation}
\begin{aligned}
\mathbb{P}(f(Y)=c_A) \ge \mathbb{P}(Y\in A)\\
\mathbb{P}(f(Y)=c_B) \le \mathbb{P}(Y\in B)\\
\end{aligned}
\end{equation}

To ensure $\mathbb{P}(f(Y)=c_A) \ge \mathbb{P}(f(Y)=c_B)$, we need:

\begin{equation}
\label{eq:50}
    \mathbb{P}(Y\in A) \ge \mathbb{P}(Y\in B)
\end{equation}

For $\mathbb{P}(Y \in A)$, we have:

\begin{equation}
\label{eq:51}
\begin{aligned}
\mathbb{P}(Y \in A) &= \iint ... \int_A k^d \exp(-\sum_{i=1}^d \frac{|x_i-\delta_i|}{\lambda_i}) dx_1 dx_2 ... dx_d \\
&=\iint ... \int_A k^d \exp(-\sum_{i=1}^d \frac{|x_i|}{\lambda_i}) \exp (-T(x)) dx_1 dx_2 ... dx_d \\
&\ge \iint ... \int_A k^d \exp(-\sum_{i=1}^d \frac{|x_i|}{\lambda_i}) \exp(-\sum_{i=1}^d \frac{|\delta_i|}{\lambda_i}) dx_1 dx_2 ... dx_d \\
&= \exp(-\sum_{i=1}^d \frac{|\delta_i|}{\lambda_i}) \ \underline{p_A}
\end{aligned}
\end{equation}
The inequality in the middle is derived by Eq. (\ref{eq:46}). Simlilarly, for $\mathbb{P}(Y \in B)$, we have:

\begin{equation}
\label{eq:52}
\begin{aligned}
\mathbb{P}(Y \in B) &= \iint ... \int_b k^d \exp(-\sum_{i=1}^d \frac{|x_i-\delta_i|}{\lambda_i}) dx_1 dx_2 ... dx_d \\
&=\iint ... \int_B k^d \exp(-\sum_{i=1}^d \frac{|x_i|}{\lambda_i}) \exp (-T(x)) dx_1 dx_2 ... dx_d \\
&\le \iint ... \int_B k^d \exp(-\sum_{i=1}^d \frac{|x_i|}{\lambda_i}) \exp(\sum_{i=1}^d \frac{|\delta_i|}{\lambda_i}) dx_1 dx_2 ... dx_d \\
&= \exp(\sum_{i=1}^d \frac{|\delta_i|}{\lambda_i}) \ \overline{p_B}
\end{aligned}
\end{equation}

Therefore, the robustness is guaranteed if $\exp(-\sum_{i=1}^d \frac{|\delta_i|}{\lambda_i}) \ \underline{p_A} \ge \exp(\sum_{i=1}^d \frac{|\delta_i|}{\lambda_i}) \ \overline{p_B}$, which is equivalent to:

\begin{equation}
\label{eq:53}
    \sum_{i=1}^d \frac{|\delta_i|}{\lambda_i} \le \frac{1}{2} \log ( \underline{p_A}/\overline{p_B})
\end{equation}

Since

\begin{equation}
\label{eq:54}
    \sum_{i=1}^d \frac{|\delta_i|}{\lambda_i} \le \sum_{i=1}^d \frac{|\delta_i|}{\min \{\lambda_i\}}
\end{equation}

if $||\delta||_1 \le \frac{1}{2} \min \{ \lambda_i\} \log ( \underline{p_A}/\overline{p_B})$, the inequality (\ref{eq:53}) is also sufficed.

Also, if we apply the complement set of A to Eq. (\ref{eq:51}), we have:

\begin{equation}
    \mathbb{P}(f(Y)=c_A) \ge 1- \exp (\sum_{i=1}^d \frac{|\delta_i|}{\lambda_i})(1-\underline{p_A})
\end{equation}

By Eq. (\ref{eq:52}) and Eq. (\ref{eq:50}), we have:

\begin{equation}
    \sum_{i=1}^d \frac{|\delta_i|}{\lambda_i} \le -\log(1-\underline{p_A}+\overline{p_B})
\end{equation}

Similarly, by Eq. (\ref{eq:54}), to guarantee the robustness, we need:

\begin{equation}
    ||\delta||_1 \le - \min \{\lambda_i\} \log(1-\underline{p_A}+\overline{p_B})
\end{equation}

In conclusion, to guarantee the robustness, we need:

\begin{equation}
    ||\delta||_1 \le \max  \{\frac{1}{2} \min \{\lambda_i\} \log (\underline{p_A}/\overline{p_B}), -\min \{\lambda_i\} \log(1-\underline{p_A}+\overline{p_B})\}
\end{equation}

This completes the proof.
\end{proof}


